\documentclass{article}

\usepackage{fullpage}
\usepackage{amsmath,amssymb,amsfonts}
\usepackage{graphicx}
\usepackage{authblk}

\usepackage[utf8]{inputenc} 
\usepackage[T1]{fontenc}    
\usepackage{hyperref}       
\usepackage{url}            
\usepackage{booktabs}       
\usepackage{amsfonts}       
\usepackage{nicefrac}       
\usepackage{microtype}      

\def\R{{\mathbb{R}}}
\def\pr{{\rm Pr}}
\def\E{{\mathbb E}}
\def\X{{\mathcal X}}

\def\F{{\mathcal F}}

\def\CC{{\mathcal C}}
\def\P{{\mathcal P}}
\def\B{{\mathcal B}}

\def\diam{{\mbox{\rm diam}}}

\newtheorem{thm}{Theorem}
\newtheorem{lemma}[thm]{Lemma}

\newenvironment{proof}{\noindent {\sc Proof:}}{$\Box$ \medskip}

\usepackage[dvipsnames]{xcolor}

\title{Expressivity of expand-and-sparsify representations}

\author[1]{Sanjoy Dasgupta}
\author[2]{Christopher Tosh}
\affil[1]{University of California, San Diego}
\affil[2]{Columbia University}

\begin{document}

\maketitle

\begin{abstract}
A simple sparse coding mechanism appears in the sensory systems of several organisms: to a coarse approximation, an input $x \in \R^d$ is mapped to much higher dimension $m \gg d$ by a random linear transformation, and is then sparsified by a winner-take-all process in which only the positions of the top $k$ values are retained, yielding a $k$-sparse vector $z \in \{0,1\}^m$. We study the benefits of this representation for subsequent learning.

We first show a universal approximation property, that arbitrary continuous functions of $x$ are well approximated by linear functions of $z$, provided $m$ is large enough. This can be interpreted as saying that $z$ unpacks the information in $x$ and makes it more readily accessible. The linear functions can be specified explicitly and are easy to learn, and we give bounds on how large $m$ needs to be as a function of the input dimension $d$ and the smoothness of the target function. Next, we consider whether the representation is adaptive to manifold structure in the input space. This is highly dependent on the specific method of sparsification: we show that adaptivity is not obtained under the winner-take-all mechanism, but does hold under a slight variant. Finally we consider mappings to the representation space that are random but are attuned to the data distribution, and we give favorable approximation bounds in this setting.
\end{abstract}

\section{Introduction}

A striking neural architecture  appears in the sensory systems of several organisms: a transformation from a low-dimensional dense representation of sensory stimulus to a much higher-dimensional, sparse representation. This has been found, for instance, in the olfactory system of the fly~\cite{W13} and mouse~\cite{SA09}, the visual system of the cat~\cite{OF04}, and the electrosensory system of the electric fish~\cite{CLM11}.

Consider, for example, the olfactory system of Drosophila~\cite{TBL08,MTJ09,W13,CRAA13}. The primary sense receptors of the fly are the roughly 2,500 odor receptor neurons (ORNs) in its antennae and maxillary palps. These can be clustered into 50 types, based on their odor responses. All ORNs of a given type converge on a corresponding glomerulus in the antennal lobe; there are 50 of these in a topographically fixed configuration, and their activations constitute a dense, 50-dimensional sensory input vector. This information is then relayed via projection neurons to a collection of roughly 2000 Kenyon cells (KCs) in the mushroom body, with each KC receiving signal from roughly 5-10 glomeruli. The pattern of connectivity between the glomeruli and Kenyon cells appears random~\cite{CRAA13}. The output of the KCs is integrated by a single anterior paired lateral (APL) neuron which then provides negative feedback causing all but the 5\% highest-firing KCs to be suppressed~\cite{LBCLM14}. The result is a sparse high-dimensional representation of the sensory input, that is the basis for subsequent learning.

\begin{figure}
\begin{center}
\includegraphics[width=3in]{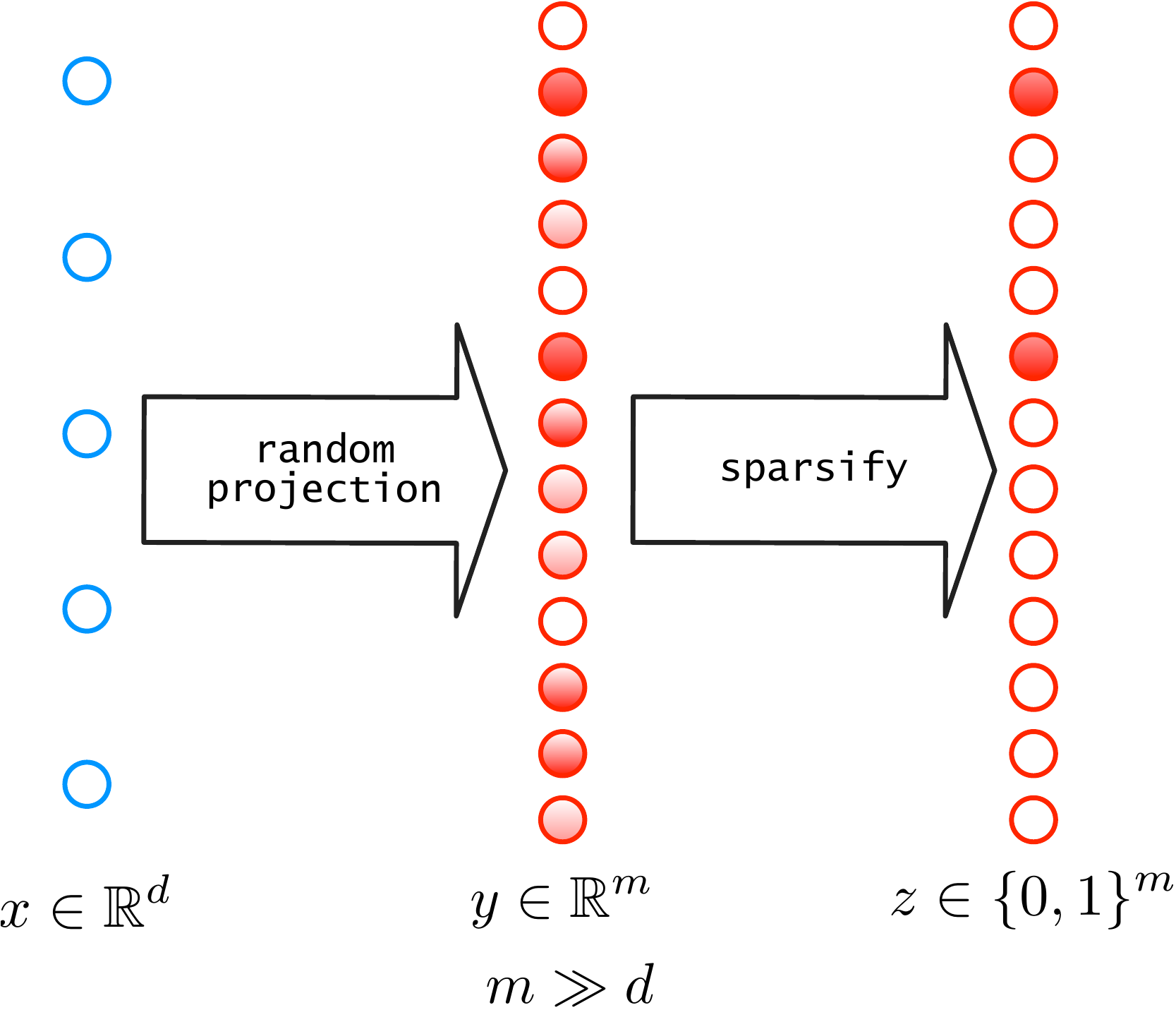}
\end{center}
\caption{The expand-and-sparsify architecture.}
\label{fig:architecture}
\end{figure}

To study the benefits of this representation, we start by modeling the process by which it is produced (Figure~\ref{fig:architecture}). Let $\X \subset \R^d$ denote a space of sensory inputs. A given $x \in \X$ is transformed in two steps.
\begin{enumerate}
\item A random linear mapping to higher dimension, $y = \Theta x \in \R^m$.

Here $\Theta$ is an $m \times d$ random matrix whose rows are drawn independently from some distribution $\nu$ over $\R^d$. For instance, $\nu$ might be the uniform distribution over the unit sphere $S^{d-1}$, or a spherical Gaussian $N(0, \sigma^2 I_d)$, or a distribution that depends on the data.

Let $\theta_j$ be the $j$th row of $\Theta$. The $j$th coordinate of $y$ is given by the dot product $y_j = \theta_j \cdot x$.

\item A sparsification operation that yields a vector $z \in \{0,1\}^m$.

This is achieved by identifying the locations of the $k$ largest entries of $y$:
$$ z_j =
\left\{
\begin{array}{ll}
1 & \mbox{if $y_j$ is one of the $k$ largest entries of $y$} \\
0 & \mbox{otherwise}
\end{array}
\right.
$$
(breaking ties arbitrarily). We refer to this as {\it $k$-winner-take-all} ($k$-WTA).

\end{enumerate}

Given the ubiquity of the expand-and-sparsify architecture, there have been a number of mathematical studies of its properties. The work of \cite{BS14} analyzed the effect of the transformation $x \rightarrow z$ on cluster structure; \cite{DSN17} showed that it is locality-preserving and that the sparse representation is well-suited to similarity search; \cite{DSSN18} showed how it can be used for novelty detection; and \cite{PV19} showed how expand-and-sparsify can be composed to solve more elaborate computational problems.

In this work, we study the benefits of the high-dimensional sparse binary representation $z$ for subsequent {\it learning}. Two particular questions interest us. (1) Does $z$ ``unpack'' the information in $x$ and make it more accessible? For instance, is it the case that any function $f(x)$ is well approximated by a {\it linear} function of $z$, if $m$ is large enough, and that this linear function is easy to learn? (2) How is manifold structure in $x$ reflected in $z$?

With regard to the first question, one basic observation is that applying a linear function to $z$ is akin to applying a two-layer neural network to $x$. Classical results assert that such networks are capable of representing arbitrary continuous functions $f$ of the input, provided the hidden layer is wide enough~\cite{C89,F89, HSW89,B93}. There is also work that quantifies the effects of picking random directions for the first layer (e.g., \cite{RR07}). A key difference in our setting is the sparsity of $z$: it is unclear if approximation results for neural nets continue to hold when hidden layers are forced to be sparse. We answer this question in the affirmative for the architecture described above (Theorem~\ref{thm:general-rate}), using techniques that are familiar from the study of nonparametric estimators such as nearest neighbor regression~\cite{DGL96}. We give an explicit expression for the resulting linear function of $z$, and it is seen to be easily learnable. We also give an upper bound on how long $z$ needs to be---that is, an upper bound on $m$---in terms of the input dimension $d$ and the smoothness of the target function.

Moving on to the second question, we look at situations where the input space $\X$ is not fully $d$-dimensional but rather lies on a lower $d_o$-dimensional manifold, as is often the case with sensory data. In such cases, the representation size ($m$) would ideally scale primarily with $d_o$ rather than $d$. We begin by showing that the architecture described above is {\it not} adaptive to manifold structure (Theorem~\ref{thm:lower-bound}). But then we consider an alternative sparsification mechanism that we call {\it $k$-thresholding}:
\begin{itemize}
\item Each coordinate $z_j$ has its own threshold $\tau_j$, and fires whenever $y_j \geq \tau_j$.
\item The value $\tau_j$ is set so that $z_j$ fires a $k/m$ fraction of the time. 
\end{itemize}
This scheme is biologically plausible, is similar in spirit to $k$-WTA, and produces representations that are $k$-sparse {\it in expectation}. However, it is slightly more adaptive to the data distribution, and this small change is enough to make it adapt, provably, to manifold structure (Theorem~\ref{thm:manifold-approx}).

In all the results above, the random linear mapping to higher dimension is not tailored in any way to the statistics of the data. It is impressive that universal approximation results hold in this situation, but one would nonetheless expect better performance if the matrix $\Theta$ were at least somewhat tuned to the input. We consider a mild form of such tuning, in which rows of $\Theta$ are chosen from roughly the same distribution as the data. It is possible, for instance, that this kind of adaptation to the sensory environment might be achievable by evolution or by a learning process. We show significantly stronger approximation bounds in this setting (Theorem~\ref{thm:manifold-distrib-approx}).

All proofs are in Section~\ref{sec:proofs}.

\section{Preliminaries}

Let $\X \subset \R^d$ denote the input space, and $\mu$ a distribution on $\X$ from which inputs are drawn.

In the fly, inputs are rescaled to have roughly the same norm~\cite{OBW10}. To see why this is necessary, notice that because of the $k$-winner-take-all operation, all vectors $\{cx: c > 0\}$ generate exactly the same $z$-representation as $x$: information about the lengths of vectors is lost.

Our analysis will use the $\ell_2$ norm, denoted $\|\cdot\|$. We will assume that inputs are normalized to have unit length: $\X \subset S^{d-1} = \{x \in \R^d: \|x\| = 1\}$. Two other pieces of norm-related notation: the open ball of radius $r$ centered at $x$ is denoted $B(x,r) = \{z \in \R^d: \|z-x\| < r\}$, and the diameter of a set $S \subset \R^d$ is defined as $\diam(S) = \sup_{x,y \in S} \|x - y\|$.

In the expand-and-sparsify architecture, an input $x \in \R^d$ is assigned a sparse binary representation $z \in \{0,1\}^m$ by applying (i) a linear map $y = \Theta x$ followed by (ii) a sparsification operation. We study two choices for each step. For (i), we assume that the rows of the $m \times d$ random matrix $\Theta$ are chosen independently at random from some distribution $\nu$. We first look at the case where $\nu$ is the uniform distribution over the $d$-dimensional unit sphere (Section~\ref{sec:basic-guarantee}), and we later consider a setting in which $\nu$ is more attuned to the statistics of the data. For (ii), vector $y \in \R^m$ is sparsified to $z$ either by $k$-winner-take-all (Section~\ref{sec:basic-guarantee}) or by $k$-thresholding (Section~\ref{sec:k-thresholding}).

\section{Approximation guarantees under winner-take-all sparsification}
\label{sec:basic-guarantee}

In this section, we study the effects of $k$-WTA sparsification.

Let $f: \X \rightarrow \R$ denote a target function that we wish to approximate using a {\it linear} function of the expanded-and-sparsified representation $z$. Since $z$ is $k$-sparse, we can write such a linear function as $(1/k) \sum_{j=1}^m w_j z_j$, for some coefficients $w_1, \ldots, w_m$. 

As long as suitable coefficients exist, they can be learned by algorithms like stochastic gradient descent, but our results in fact hold for an {\it explicit} choice of the $\{w_j\}$. For any $1 \leq j \leq m$, let $C_j$ be the set of inputs $x$ that cause $z_j$ to fire,
$$ C_j = \{x \in \X: z_j = 1\} .$$
This set depends on the random matrix $\Theta$. Now, let $w_j$ be {\it the average value of $f$ in region $C_j$}. To make this precise, for any measurable $S \subset \X$ with $\mu(S) > 0$, define
$f(S) = (1/\mu(S)) \int_S f \, d \mu $.
Then $w_j = f(C_j)$. Notice that this weight is simply the expected value of $z_j f(x)$ and can thus be learned via a Hebbian rule.

Let $\widehat{f}$ denote the resulting function,
\begin{equation}
\widehat{f}(x) = \frac{1}{k} \sum_{j=1}^m w_j z_j = \frac{1}{k} \sum_{j: x \in C_j} f(C_j) .
\label{eq:wta-fn}
\end{equation}
We will bound the discrepancy between $f$ and $\widehat{f}$.

\subsection{Cell diameters}

Classical results in statistics show that it is not possible to give rates of convergence for nonparametric estimation without conditions on the target function~\cite{S80}. We will make the common assumption that there is a constant $\lambda > 0$ such that $f$ is $\lambda$-Lipschitz with respect to $\ell_2$ norm,
$$f(x) - f(y) \leq \lambda \|x-y\| \mbox{\ \ for all $x,y \in \X$.}$$
We can then bound the approximation error of $\widehat{f}$ in terms of the diameters of the cells $C_j$.
\begin{lemma}
If $f$ is $\lambda$-Lipschitz, and $\widehat{f}$ is defined as in (\ref{eq:wta-fn}), then
$$ \sup_{x \in \X} \left|\widehat{f}(x) - f(x) \right| \ \leq \ \frac{\lambda}{k} \sum_{j: x \in C_j} \diam(C_j) .$$
\label{lemma:wta-approx-bd}
\end{lemma}
The diameters of cells depend upon the random matrix $\Theta$. Recall that the rows of this matrix, $\theta_1, \ldots, \theta_m$, are drawn independently at random from some distribution $\nu$ over $\R^d$. 
We will take $\nu$ to be a distribution over {\it unit vectors} in $\R^d$. On any input $x$, the $k$ coordinates of $z$ that fire are those with the highest values of $y_j = \theta_j \cdot x$, or equivalently those for which $\theta_j$ is closest to $x$:
$$ y_j \geq y_i
\ \Leftrightarrow \ 
\theta_j \cdot x \geq \theta_i \cdot x
\ \Leftrightarrow \ 
\|x - \theta_j\| \leq \|x - \theta_i\| .$$
Intuitively speaking, the cells $C_j$ are {\it local}: the only inputs that make $z_j$ fire are those that lie near $\theta_j$. We now quantify this. For any $r > 0$, define
\begin{equation}
\nu_o(r) = \inf_{x \in \X} \nu(B(x,r)) 
\label{eq:nu-min}
\end{equation}
to be the smallest probability mass of any ball of radius $r$ centered in $\X$. Then the regions $C_j$ can be shown to have the following locality property.
\begin{lemma}
There is an absolute constant $c_o > 0$ for which the following holds. Pick any $0 < \delta < 1$. With probability at least $1-\delta$,
$$ \max_j \diam(C_j) \leq 2 \nu_o^{-1} \left( \frac{2k}{m} + \frac{2c_o (d \log m + \log 1/\delta)}{m} \right) .$$
\label{lemma:C-diam-bound}
\end{lemma}

For instance, if $k \gg \log m$, or more specifically,
\begin{equation}
k \geq c_o \left(d \log m + \log \frac{1}{\delta} \right)
\label{eq:k-lower-bound}
\end{equation}
then the maximum cell diameter is at most $2 \nu_o^{-1}(4k/m)$.

\subsection{A general approximation bound}

The analysis above gives a basic bound in the case where $\nu$ is uniform over unit vectors.
\begin{thm}
Suppose the expand-and-sparsify process uses (i) a random mapping $\Theta$ based on distribution $\nu$ which is uniform over $S^{d-1}$ and (ii) $k$-winner-take-all sparsification. Pick any $0 < \delta < 1$. If $f$ is $\lambda$-Lipschitz, then with probability at least $1-\delta$ over the choice of $\Theta$, the approximating function $\widehat{f}$ from (\ref{eq:wta-fn}) satisfies
$$ \sup_{x \in \X} \left|\widehat{f}(x) - f(x) \right| \ \leq \ c_d \lambda \left( \frac{2k}{m} + \frac{2c_o (d \log m + \log 1/\delta)}{m} \right)^{1/(d-1)} ,$$
where $c_d$ is a constant that depends on the dimension $d$.
\label{thm:general-rate}
\end{thm}
If $k$ satisfies (\ref{eq:k-lower-bound}), then the bound simplifies to $c_d \lambda (4k/m)^{1/(d-1)}$. And for any $k \ll m$, taking $m \rightarrow \infty$ yields an arbitrarily good approximation of $f$ by a linear function of $z$. 

Theorem~\ref{thm:general-rate} is quite different in flavor from the well-known universal approximation theorems for neural nets with a single hidden layer~\cite{C89,F89, HSW89,B93}. It (i) requires no learning in the mapping from the input to the hidden layer, (ii) allows the hidden layer to be as sparse as desired, (iii) has a simple form for the weights from the hidden to output layer, and (iv) gives explicit bounds on the required size of the hidden layer. It is much closer, in both conclusion and proof technique, to classical results in nonparametric estimation which study asymptotic convergence of kernel regression and nearest neighbor methods~\cite{DGL96}. Of these, the connection to nearest neighbor is particularly strong: indeed, the prediction at any point $x$ depends on its $k$ nearest neighbors amongst the $\theta$'s, which can be thought of as surrogate ``training points''. Curiously, while results on nearest neighbor consistency have traditionally required the training data to be from the same distribution as test data~\cite{CH67,S77,CD14}, we get convergence with $\theta$'s that are unrelated to the data distribution.

\subsection{Adaptivity to manifold structure?}
\label{sec:not-adaptive}

The approximation bound of Theorem~\ref{thm:general-rate} is typical for nonparametric estimation, but scales poorly with $d$, the dimension of the input. What if the input space $\X \subset \R^d$ lies on a submanifold of dimension $d_o \ll d$, as is often hypothesized to be the case with sensory data? We would then hope for a bound in which the dependence on $d$ (in the exponent) is replaced by $d_o$.

Sadly, we do not get adaptivity to intrinsic dimension when the rows of the random mapping $\Theta$ are chosen uniformly at random from $S^{d-1}$ and winner-take-all sparsification is used. To see why, consider a simple example in which the data lie on a one-dimensional manifold, a circle in $S^{d-1}$:
\begin{equation}
\X_1 = \{(x_1, x_2, 0, 0, \ldots, 0) \in \R^d: x_1^2 + x_2^2 = 1\}.
\label{eq:circle}
\end{equation}
Since we are primarily interested in how approximation error scales with $m$, let's take $k=1$. As in the proof of Theorem~\ref{thm:general-rate}, with high probability every $x \in \X_1$ will have some $\theta_j$ within distance $1/m^{1/(d-1)}$, roughly. This means that any $\theta_j$ further than this from $\X_1$ is unused: its corresponding $z_j$ will {\it never} get activated! The number of $\theta$'s that are actually used turns out to be roughly $m^{1/(d-1)}$, leading to the same error rate as in Theorem~\ref{thm:general-rate}, despite the low dimensionality of the data.

\begin{thm}
For any $d > 3$, let input space $\X_1$ be the one-dimensional submanifold of $\R^d$ given in (\ref{eq:circle}). Take $k=1$. Suppose that random matrix $\Theta$ has rows chosen from the distribution $\nu$ that is uniform over $S^{d-1}$. For any $0 < \lambda < 1$, there exists a $\lambda$-Lipschitz function $f: \X_1 \rightarrow \R$ such that with probability at least $1/2$ over the choice of $\Theta$, no matter how the weights $w_1, \ldots, w_m$ are set, the resulting function $\widehat{f}$ has approximation error at least
$$ \sup_{x \in \X_1} \left| \widehat{f}(x) - f(x) \right| \ \geq \ c_d' \cdot \lambda \cdot \frac{1}{m^{1/(d-1)} \log m},$$
where $c_d'$ is some absolute constant depending on $d$.
\label{thm:lower-bound}
\end{thm}
Thus, picking an expansion map that is entirely oblivious to the data distribution and using winner-take-all sparsification does yield a universal approximation guarantee, but is not adaptive to intrinsic low-dimensional structure. In cases where the data lie near a low-dimensional manifold, only a tiny fraction of the expanded representation is ever used. 

We will see, however, that a slight change in the method of sparsification gives much better results.

\section{An alternative sparsification scheme}
\label{sec:k-thresholding}

We consider an alternative process in which each $z_j$ has its own threshold $\tau_j$, chosen so that $z_j$ fires a fraction $k/m$ of the time. We call this {\it $k$-thresholding}. More precisely, for $j = 1, \ldots, m$, let $y_j = \theta_j \cdot x$ as before (where $\theta_j$ is the $j$th row of random matrix $\Theta$), and then set
$$ z_j =
\left\{
\begin{array}{ll}
1 & \mbox{if $y_j \geq \tau_j$} \\
0 & \mbox{otherwise}
\end{array}
\right.
$$
Here $\tau_j = \tau(\theta_j)$, where the function $\tau: \R^d \rightarrow \R$ is defined by
\begin{equation}
\pr_{x \sim \mu} \left( \theta \cdot x \geq \tau(\theta) \right) = \frac{k}{m} .
\label{eq:tau}
\end{equation}
(Recall that $\mu$ is the distribution over the inputs $x$.) It might not be possible to achieve this equality if, for instance, $x$ is discrete; in that case we ask for the supremum over $\tau$-values for which the left-hand side is at least $k/m$. However, we will ignore this technicality and assume that the equality above is achieved. By linearity of expectation,
$\E_{x \sim \mu} [\#(\mbox{ones in $z$})] = k$.
Thus, the $z$-vectors produced in this way are $k$-sparse {\it in expectation}.

Recall from the discussion in Section~\ref{sec:not-adaptive} that winner-take-all sparsification can waste directions $\theta_j$, with some $z_j$ never firing at all. On the other hand, $k$-thresholding ensures that all $z_j$ are used.

\subsection{Response regions}

As before, we approximate a function $f: \X \rightarrow \R$ by a linear function of $z$, defined by weights $w_1, \ldots, w_m$. We will take this approximation to be the average over all $w_j$ for which $z_j$ is firing,
\begin{equation}
\widehat{f}(x) = \frac{\sum_j w_j z_j}{\sum_j z_j} ,
\label{eq:f-hat-thresholding}
\end{equation}
whenever the denominator is positive. Our analysis isn't too finicky about the precise form of this function; for instance, our bounds hold if we instead return {\it any} $w_j$ for which $z_j = 1$.

Once again, we take $C_j$ to be the set of points $x \in \X$ that cause $z_j$ to fire. Thus $C_j = C(\theta_j)$, where
\begin{equation}
C(\theta) = \{x \in \X: \theta \cdot x \geq \tau(\theta)\} ,
\label{eq:response-region}
\end{equation}
for $\tau(\cdot)$ as defined in (\ref{eq:tau}). 

With winner-take-all sparsification, we showed that the regions $C_j, 1 \leq j \leq m$, are necessarily local and thus encompass narrow ranges of $f$-values. Under thresholding, locality is no longer assured. Instead, we identify a subset of {\it good} regions that are local and show that there are enough of these. Such regions will be assigned weight $w_j = f(C_j)$, while other regions will get weight zero.

By copying the proof of Lemma~\ref{lemma:wta-approx-bd}, we get the following counterpart.
\begin{lemma}
Under the $k$-thresholding scheme, if $f$ is $\lambda$-Lipschitz, then for all $x \in C_1 \cup \cdots \cup C_m$,
$$ \left| \widehat{f}(x) - f(x) \right| \leq \lambda \cdot \max_{j: x \in C_j} \diam(C_j) .$$
\label{lemma:threshold-approx-bd}
\end{lemma}
We will see that when the input space $\X$ lies on a low-dimensional manifold in $\R^d$, the diameters of the response regions $C_j$ can be bounded in terms of the manifold dimension $d_o$. 

\subsection{The manifold assumption}

We now identify the input space $\X$ with a compact $d_o$-dimensional Riemannian submanifold $M$ of $\R^d$ that happens to be contained in the unit sphere, that is, $M \subset S^{d-1}$. We will assume that $M$ has nice boundaries and that the distribution on it, $\mu$, is almost-uniform: formally, there exist constants $c_1, c_2, c_3 > 0$ such that for all $x \in M$ and for all $r \leq r_o$ (where $r_o$ is some absolute constant),
\begin{gather}
c_1 r^{d_o} \leq \mu(B_M(x,r)) < c_2 r^{d_o} \label{eq:almost-uniform-mu} \\
\mbox{vol}(B_M(x,r)) \geq c_3 r^{d_o} .
\label{eq:volume-on-manifold}
\end{gather}
Here $B(x,r)$ is the open Euclidean ball of radius $r$ centered at $x$, and $B_M(x,r) = B(x,r) \cap M$.

In order to analyze data distributions supported on a manifold, it is necessary to impose conditions on the curvature. We adopt the common requirement that $M$ has positive {\it reach} $\rho > 0$: that is, every point in an open tubular neighborhood of $M$ of radius $\rho$ has a unique nearest neighbor in $M$~\cite{NSW06}. For instance, a $d$-dimensional sphere of radius $r$ has reach $r$ since any point in $\R^d$ at distance $< r$ from the sphere has a unique nearest point---that is, a unique projection---on the sphere. 

We can relate the reach to normal bundles on $M$. For any $x \in M$, let $N(x)$ denote the $(d-d_o)$-dimensional subspace of normal vectors to the tangent plane at $x$. Then the sets $\Gamma_\rho(x) = \{x + r u \in \R^d: u \in N(x), \|u\|=1, 0 < r < \rho\}$, over $x \in M$, are disjoint. Denote their union by $\Gamma_\rho$, and let $\pi_M: \Gamma_\rho \rightarrow M$ be the projection map that sends any point in $\Gamma_\rho(x)$ to $x$, its nearest neighbor in $M$.

Suppose, for example, that $M$ is the set from (\ref{eq:circle}), a unit circle within the $d$-dimensional sphere: $M = S^2 \times \{0^{d-2}\}$. Pick any $x = (x_1, x_2, 0, \ldots, 0) \in M$, with $x_1^2 + x_2^2 = 1$. The hyperplane normal to $M$ at $x$ can be written $N = \{z \in \R^d: z_1 x_2 - z_2 x_1 = 0\}$. Then, any point $z \in N$ at distance $<1$ from $x$ has $x$ as its nearest neighbor in $M$. The set of all these points is $\Gamma_1(x)$, and this manifold $M$ has reach 1.

\subsection{Bounding cell diameters}

Matrix $\Theta$ has rows sampled from a distribution $\nu$. In what follows, we take $\nu$ to be the multivariate Gaussian $N(0, \sigma^2 I_d)$, for any $\sigma > 0$. Because of concentration effects, this is rather like a uniform distribution over a sphere of radius $\sigma \sqrt{d}$, but assigns non-zero density to every point in $\R^d$.

Let $\theta_1, \ldots, \theta_m$ be the rows of matrix $\Theta$. We call $\theta \in \R^d$ {\it good} if it lies in $\Gamma_{\rho/2}$. As we will see, good $\theta_j$'s are close enough to the manifold $M$ that they are guaranteed to be activated by a single neighborhood of $M$. This makes them useful in obtaining an accurate approximation to the target function $f$. The remaining (non-good) $\theta_j$'s can potentially be activated by different regions of $M$ in which $f$ behaves differently; to prevent them from corrupting the approximation, we set the corresponding weights to zero.

Recalling the definition (\ref{eq:response-region}) of response regions $C_j$, define weights $w_1, \ldots, w_m$ by 
$$
w_j =
\left\{
\begin{array}{ll}
f(C_j) & \mbox{if $\theta_j$ is good} \\
0 & \mbox{otherwise .}
\end{array}
\right.
$$
The following lemma captures the intuition that good $\theta$'s have a local response region.
\begin{lemma}
Pick any good $\theta \in \R^d$. Define $\Delta = \|\theta - \pi_M(\theta)\|$ to be the distance from $\theta$ to its projection in $M$. Let $C(\theta)$ be the response region associated with $\theta$, as defined in (\ref{eq:response-region}). Then
$$ B_M\left(\pi_M(\theta), \sqrt{\frac{\rho - \Delta}{\rho + \Delta}}\left(\frac{k}{c_2 m} \right)^{1/d_o} \right) \ \subset \  
C(\theta) \ \subset \ B_M\left(\pi_M(\theta),  \sqrt{\frac{\rho + \Delta}{\rho - \Delta}} \left(\frac{k}{c_1 m} \right)^{1/d_o} \right) $$
provided $(k/(c_1m))^{1/d_o} < \min(\rho, r_o)$.
\label{lemma:good-theta}
\end{lemma}

Since any good $\theta$ has distance $\Delta < \rho/2$ to the manifold, we immediately get the bound
$$ \diam(C(\theta)) \leq 4 \left(\frac{k}{c_1 m} \right)^{1/d_o} $$
provided $(k/(c_1m))^{1/d_o} < \min(\rho, r_o)$.
Given Lemma~\ref{lemma:threshold-approx-bd}, it then remains to be shown that for every $x \in M$, there is some good $\theta_j$ such that $x$ lies in $C(\theta_j)$. That is, the good cells cover all of $M$.

\begin{lemma}
Pick $\theta_1, \ldots, \theta_m \sim \nu$. There is a constant $c_d'$, depending on $d$, for which the following holds. Pick any $0 < \delta < 1$. Set $k \geq c_d' \ln (m/\delta)$. Then with probability at least $1 - \delta$ over the choice of $\theta_j$'s: for every $x \in M$ there is a good $\theta_j$ with $x \in C(\theta_j)$.
\label{lemma:good-theta-cover}
\end{lemma}

Putting these pieces together yields the following approximation result.
\begin{thm}
Suppose the data distribution is supported on a $d_o$-dimensional submanifold $M$ of $R^d$ with reach $\rho > 0$, that additionally satisfies conditions (\ref{eq:almost-uniform-mu}) and (\ref{eq:volume-on-manifold}). Suppose also that the rows of matrix $\Theta$ are chosen from $N(0, \sigma^2 I_d)$, for some $\sigma > 0$. There is a constant $c_d'$, depending on the dimension $d$, for which the following holds. Pick $0 < \delta < 1$. Let $k$ and $m$ be chosen so that $k \geq c_d' \ln (m/\delta)$ and $(k/(c_1m))^{1/d_o} \leq \min(\rho, r_o)$. Then with probability at least $1-\delta$ over the choice of $\Theta$, we have that the approximating function $\widehat{f}$ of (\ref{eq:f-hat-thresholding}) satisfies
$$ \sup_{x \in \X} \left| \widehat{f}(x) - f(x) \right| \leq 4 \lambda \left( \frac{k}{c_1 m}\right)^{1/d_o} .$$
\label{thm:manifold-approx}
\end{thm}

Unlike Theorem~\ref{thm:general-rate}, which holds for arbitrary $k$, here we require $k \gg \log m$. With winner-take-all sparsification, every point $x \in M$ necessarily activates $k$ of the $z$-units, but the danger is that some of these $z$-units might never be used. With $k$-thresholding, every $z$-unit does get used, but there is a new danger, that some points $x \in M$ might not activate any of them. It is to banish this possibility that the lower bound on $k$ is needed.

\section{A data-dependent expansion mapping}

To this point, we have considered settings where the rows of $\Theta$ are chosen without regard for the underlying data. While $k$-thresholding was shown to be adaptive to manifold structure, we also saw that $k$-winner-take-all sparsification is not.

In this section, we turn our attention back to winner-take-all, but under the assumption that the rows of the matrix $\Theta$ are selected in a data-dependent fashion. Specifically, we will look at settings where $\theta_1, \ldots, \theta_m$ are drawn i.i.d. from a distribution $\nu$ whose support satisfies
$\text{supp}(\nu) = M = \text{supp}(\mu)$, where $M$ is a smooth $d_o$-dimensional manifold. (Recall that $\mu$ is the distribution over inputs.) That is, the $\theta$'s take values on the same set as the data. 

The only assumption that we will make on $\nu$ is that it is sufficiently uniform over $M$, i.e., that there exist constants $c_o, r_o > 0$ such that
\begin{equation}
\label{eqn:measure-lower-bound}
\nu(B(x,r)) \geq c_o r^{d_o}
\end{equation}
whenever $r \leq r_o$ for all $x \in M$. Under this assumption, we have the following approximation result. 

\begin{thm}
\label{thm:manifold-distrib-approx}
Let $c_o, r_o >0$ obey equation~\eqref{eqn:measure-lower-bound}. Pick $0 < \delta < 1$, and set $k$ and $m$ so that
\[ \frac{2k}{m} + \frac{d_o \log(8m/k) + \log(4/\delta)}{m} \ \leq \ c_o(4/r_o)^{d_o}. \] 
Then with probability $1-\delta$ over the random choice of $\Theta$, we have that the approximating function $\widehat{f}$ from (\ref{eq:wta-fn}) satisfies the bound
\[ \sup_{x \in M}|f(x) - \widehat{f}(x)| \ \leq \ \frac{8\lambda}{c_o^{1/d_o}} \left( \frac{2k}{m} + \frac{d_o \log(8m/k) + \log(4/\delta)}{m} \right)^{1/d_o} . \]
\end{thm}

While the rate is essentially the same as the one given in Theorem~\ref{thm:manifold-approx}, there are two key differences. First, $k$ is allowed to be as small as desired; indeed the optimal setting of $k$ in the upper bound is $d_o/2$. Second, there is no dependence on the ambient dimension $d$, even in the constants.

\section{Discussion}

In this paper we study a representation of data inspired by the neural architecture in the sensory systems of several organisms. There are some high-level choices to be made in the representation: (1) whether the expansion map (the random linear map $\Theta$) is oblivious to the data distribution or adaptive to it in some way, and (2) the manner in which the resulting high-dimensional encoding $\Theta x$ is sparsified. We show that under all these settings, the representation has a universal approximation property. However, if the data $x$ lie near a low-dimensional manifold, then the approximation error can vary dramatically depending upon the method of sparsification. In particular, allowing each unit of the encoding to have its own threshold is preferable to a winner-take-all mechanism.

We then look at expansion maps $\Theta$ whose rows are sampled from a distribution similar to that of the data, and find that this mild data-dependence leads to significantly better approximation. 

One intriguing open problem is to look at more sophisticated forms of adaptivity, for instance based on a clustering of the data. A second open problem is to study expansion maps that are sparse in the sense that each unit $z_j$ receives input from only a constant number of coordinates of the input $x$; this is the case, for instance, in the fly's olfactory system.

Our work has connections to many well-known results in computer science and statistics. For instance, {\it compressed sensing}~\cite{D06,CRT06} recovers a sparse vector given random projections of it, while we use random projections to build a sparse representation; it would be interesting to understand the relationship between these two enterprises. {\it Random Fourier features}~\cite{RR07} are based on a circuit similar to ours, but in which the choice of random directions is informed by a target notion of similarity in the input space (the kernel function); perhaps our results can also be cast in this light. Finally, {\it hyperdimensional computing}~\cite{K09b} is an alternative paradigm for computer architecture in which all inputs get mapped to high-dimensional vectors, usually via random linear maps, as a prelude to computation; our findings could serve as a basis for further theoretical developments in that field.

\section{Proofs}
\label{sec:proofs}

\subsection{Proof of Lemma~\ref{lemma:wta-approx-bd}}

Pick any $x \in \X$ and any region $C_j$ containing $x$. By the Lipschitz condition, the average value of $f$ in $C_j$ can differ from $f(x)$ by at most $\lambda \, \diam(C_j)$. Thus,
$$ \widehat{f}(x) = \frac{1}{k} \sum_{j: x \in C_j} f(C_j) = \frac{1}{k} \sum_{j: x \in C_j} (f(x) + (f(C_j) - f(x))) = f(x) \pm \frac{\lambda}{k} \sum_{j: x \in C_j} \diam(C_j) .$$

\subsection{Proof of Lemma~\ref{lemma:C-diam-bound}}

A key part of this proof involves assessing how many of the rows $\theta_1, \ldots, \theta_m$ fall within specific Euclidean balls. The following result from \cite[Lemma 16]{CD10} is helpful.
\begin{lemma}[\cite{CD10}]
There is an absolute constant $c_o > 0$ for which the following holds. Pick any $0 < \delta < 1$. Pick $\theta_1, \ldots, \theta_m$ independently at random from a distribution $\nu$ on $\R^d$. Then with probability at least $1-\delta$, any ball $B$ in $\R^d$ with
$$ \nu(B) \geq \frac{2}{m}\left( k + c_o \left( d \log m + \log \frac{1}{\delta} \right)\right)$$
contains at least $k$ of the $\theta_i$.
\label{lemma:ball-concentration}
\end{lemma}

Recall that we take $\nu$ to be a distribution over unit vectors, and define the quantity $\nu_o$ as in (\ref{eq:nu-min}). Lemma~\ref{lemma:C-diam-bound} is then a direct consequence of the following.

\begin{lemma}
Pick any $0 < \delta < 1$. Suppose that $r$ is chosen so that
$$ \nu_o(r) \ \geq \ \frac{2}{m} \left( k + c_o \left( d \log m + \log \frac{1}{\delta} \right)\right),$$
where $c_o$ is the constant from Lemma~\ref{lemma:ball-concentration}. Then with probability at least $1-\delta$ over the choice of $\Theta$, we have $C_j \subset B(\theta_j, r)$ for all $1 \leq j \leq m$.
\label{lemma:C-locality}
\end{lemma}
\begin{proof}
Any ball $B$ of radius $r$ centered in $\X$ has $\nu(B) \geq \nu_o(r)$. Given the lower bound on $\nu_o(r)$, we can apply Lemma~\ref{lemma:ball-concentration} to conclude that with probability at least $1-\delta$ every such ball contains at least $k$ of the $\theta_j$'s. Thus for any $x \in \X$, its $k$ nearest $\theta_j$'s lie within radius $r$; or conversely, any $z_j$ is only activated for points $x$ within distance $r$ of $\theta_j$.
\end{proof}

\subsection{Proof of Theorem~\ref{thm:general-rate}}

Since $\X \subset S^{d-1}$, any set $B(x,r)$ with $x \in \X$ is a spherical cap, so that $\nu(B(x,r))$ scales as $r^{d-1}$. Thus, by Lemma~\ref{lemma:C-diam-bound}, we have
$$ \max_j \diam(C_j) \leq c_d \left( \frac{2k}{m} + \frac{2c_o (d \log m + \log 1/\delta)}{m} \right)^{1/(d-1)}$$
for some constant $c_d$ that depends on $d$. We then invoke Lemma~\ref{lemma:wta-approx-bd}. 

\subsection{Proof of Theorem~\ref{thm:lower-bound}}

Define $f: \X_1 \rightarrow \R$ to be a triangular function: for $0 < \theta \leq 2\pi$,
$$ 
f(\cos \theta, \sin \theta, 0, 0, \ldots, 0) 
\ = \ 
\left\{
\begin{array}{ll}
2\lambda \theta/\pi & \mbox{if $\theta \leq \pi$} \\
2\lambda (2 \pi - \theta)/\pi & \mbox{if $\theta > \pi$}
\end{array}
\right.
$$ 
This function is $\lambda$-Lipschitz.

Let $\theta_1, \ldots, \theta_m$ be chosen independently from the uniform distribution $\nu$ on $S^{d-1}$. Since $k=1$, the cell $C_j$ consists of points $x \in \X_1$ that are closer to $\theta_j$ than to any other $\theta_i$. This is a convex set, and hence every $C_j$ is either empty or consists of an arc of $\X_1$. We'll show that almost all of them are empty, and hence at least one arc has significant length, implying the bound.

By Lemma~\ref{lemma:C-locality}, there is a constant $c_1 > 0$ such that with probability $> 3/4$, any ball $B$ of probability mass $\nu(B) > (c_1 d \log m)/m$ will contain at least one $\theta_j$. Call this good event $E_1$. 

For $x \in \X_1$, the ball of radius $r$ around $x$ can be seen (Lemma~\ref{lemma:spherical-caps}(a)) to have probability mass at least
$$ \frac{1}{3 \sqrt{d}} \cdot r^{d-1} \cdot \left( 1 - r^2/4 \right)^{(d-1)/2} .$$
Choose $r_o$ so that 
$$ r_o^{d-1} \left(1 - r_o^2/4 \right)^{(d-1)/2} = \frac{3c_1 d^{3/2} \log m}{m}.$$
Under $E_1$, we then have that $B(x,r_o)$ contains at least one $\theta_j$ for every $x \in \X_1$. In particular, this means that any $\theta_j$ at distance $> r_o$ from $\X_1$ has an empty cell $C_j$ and is thus unused.

Denote the region of possibly-useful $\theta$'s by 
$$U = \{\theta \in S^{d-1}: \mbox{$\theta$ is within distance $r_o$ of $\X_1$}\} .$$ 
A simple calculation (Lemma~\ref{lemma:spherical-caps}(b)) shows that
$$
\nu(U) 
\ = \ 
\frac{1}{2} r_o^{d-2} \left(1 - r_o^2/4 \right)^{(d-2)/2} 
\ = \ 
\frac{1}{2} \left(\frac{3c_1 d^{3/2} \log m}{m} \right)^{(d-2)/(d-1)} .
$$
The expected number of $\theta_j$ that fall in $U$ is $m \nu(U)$. With probability at least $3/4$, the actual number of such $\theta_j$ is at most $4 m \nu(U)$; call this event $E_2$.

Thus with probability $\geq 1/2$, both $E_1$ and $E_2$ occur. In this case, at most $4m \nu(U) \leq 6 c_1 m^{1/(d-1)} d^{3/2} \log m$ of the $\theta_j$ have non-empty cells $C_j$. Given that the circumference of $\X_1$ is $2 \pi$, this means that at least one of the cells $C_j$ contains an arc of $\X_1$ of length $\Omega(1/(m^{1/(d-1)} \log m))$, ignoring constants in $d$. Every point in this arc will receive the same prediction under $\widehat{f}$, namely $w_j$. No matter how this is chosen, at some point on the arc the error will be at least $\lambda/2$ times the arc length.

\begin{lemma}
Let $\nu$ denote the uniform distribution on $S^{d-1}$, for $d > 3$.
\begin{enumerate}
\item[(a)] For any $x \in S^{d-1}$ and any $0 < r < 1$, we have 
$$\nu(B(x,r)) \geq \frac{1}{3\sqrt{d}} \cdot r^{d-1} \cdot (1 - r^2/4)^{(d-1)/2}.$$
\item[(b)] Let $\X_1$ denote the circle in (\ref{eq:circle}). For any $0 < r < 1$, the region of $S^{d-1}$ at distance $\leq r$ from $\X_1$ has probability mass 
$$\frac{1}{2} \cdot r^{d-2} \cdot (1-r^2/4)^{(d-2)/2}.$$
\end{enumerate}
\label{lemma:spherical-caps}
\end{lemma}
\begin{proof}
We can generate a random sample $\theta = (\theta_1, \ldots, \theta_d) \sim \nu$ by first drawing $Y_1, \ldots, Y_d$ independently from a standard normal distribution, and then taking
$$ \theta_i = \frac{Y_i}{(Y_1^2 + \cdots + Y_d^2)^{1/2}} .$$
This works because the distribution of $Y = (Y_1, \ldots, Y_d)$ is spherically symmetric. Now, for any $k$, the sum $Y_1^2 + \cdots + Y_k^2$ has a chi-squared distribution with $k$ degrees of freedom, denoted $\chi^2(k)$. It is well-known that if $A \sim \chi^2(k)$ and $B \sim \chi^2(\ell)$ are independent, then $A/(A+B)$ has a $\mbox{Beta}(k/2, \ell/2)$ distribution. Thus $\theta_1^2 = Y_1^2/(Y_1^2 + \cdots + Y_d^2)$ follows a $\mbox{Beta}(1/2, (d-1)/2)$ distribution and $\theta_1^2 + \theta_2^2 = (Y_1^2 + Y_2^2)/(Y_1^2 + \cdots + Y_d^2)$ follows a $\mbox{Beta}(1,(d-2)/2)$ distribution. 

For part (a), we can assume without loss of generality that $x = e_1$. Then 
$$ \theta \in B(x,r) 
\ \Leftrightarrow \  
\|\theta - e_1 \|^2 \leq r^2 
\ \Leftrightarrow \ 
\theta_1 \geq 1 - r^2/2 ,
$$
whereupon, using Lemma~\ref{lemma:beta-bounds} with $\epsilon = 1 - (1-r^2/2)^2 = r^2(1-r^2/4)$,
\begin{align*}
\nu(B(x,r))
&=
\frac{1}{2} \cdot \pr \left(\theta_1^2 \geq (1 - (r^2/2))^2 \right)
\ = \ 
\frac{1}{2} \cdot \pr(\theta_1^2 \geq 1- \epsilon) 
\\
&\geq  
\frac{1}{2} \cdot \frac{1}{B(1/2, (d-1)/2)} \cdot \frac{\epsilon^{(d-1)/2}}{(d-1)/2} \\
&=
\frac{1}{2} \cdot \frac{\Gamma(d/2)}{\Gamma(1/2) \Gamma((d-1)/2) (d-1)/2} \cdot r^{d-1} \cdot (1-r^2/4)^{(d-1)/2} \\
&=
\frac{1}{2} \cdot \frac{\Gamma(d/2)}{\sqrt{\pi}\Gamma((d+1)/2)} \cdot r^{d-1} \cdot (1-r^2/4)^{(d-1)/2} \\
&\geq \frac{1}{2} \cdot \frac{1}{\sqrt{\pi (d+1)/2}} \cdot r^{d-1} \cdot (1-r^2/4)^{(d-1)/2}
\ \geq \ 
\frac{1}{3 \sqrt{d}} \cdot r^{d-1}\cdot (1-r^2/4)^{(d-1)/2}
\end{align*}
where we have used $\Gamma(x+1) = x \Gamma(x)$ and $\Gamma(x-0.5)/\Gamma(x) \geq 1/\sqrt{x}$.

Now we move to part (b). For any $\theta \in S^{d-1}$, the nearest point in $\X_1$ is $(\theta_1, \theta_2, 0, \ldots, 0)/\sqrt{\theta_1^2 + \theta_2^2}$. Let $U \subseteq S^{d-1}$ denote the region of the unit sphere that is within distance $r$ of $\X_1$. Then
$$
\theta \in U 
\ \Leftrightarrow \ 
\left\| \theta - \frac{(\theta_1, \theta_2, 0, \ldots, 0)}{\sqrt{\theta_1^2 + \theta_2^2}} \right\|^2 \leq r^2 
\ \Leftrightarrow \ 
\sqrt{\theta_1^2 + \theta_2^2} \geq 1 - r^2/2 .
$$
Recalling that $\theta_1^2 + \theta_2^2$ has a $\mbox{Beta}(1,(d-2)/2)$ distribution, and using the same setting of $\epsilon$ as before, we apply Lemma~\ref{lemma:beta-bounds} to get that
\begin{align*}
\nu(U) 
&= 
\frac{1}{2} \cdot \pr(\theta_1^2 + \theta_2^2 \geq (1 - r^2/2)^2) 
\ = \ 
\frac{1}{2} \cdot \pr(\theta_1^2 + \theta_2^2 \geq 1 - \epsilon) \\
&= \frac{1}{2} \cdot \epsilon^{(d-2)/2} 
\ = \ 
\frac{1}{2} \cdot r^{d-2} \cdot (1- r^2/4)^{(d-2)/2} .
\end{align*}
\end{proof}

\begin{lemma}
Suppose $Z$ has a $\mbox{Beta}(\alpha, \beta)$ distribution with $\alpha \leq 1$ and $\beta \geq 1$. For any $0 < \epsilon < 1$,
$$ \frac{1}{B(\alpha, \beta)} \cdot \frac{\epsilon^\beta}{\beta} 
\ \leq \ 
\pr(Z \geq 1 - \epsilon) 
\ \leq \ 
\frac{1}{B(\alpha, \beta)} \cdot \frac{\epsilon^\beta}{\beta} \cdot (1-\epsilon)^{\alpha-1} 
.$$
In particular, for $\alpha = 1$ we get $\pr(Z \geq 1 - \epsilon) = \epsilon^\beta$.
\label{lemma:beta-bounds}
\end{lemma}
\begin{proof}
Recall that $Z$ has density
$$ p(z) = \frac{1}{B(\alpha, \beta)} z^{\alpha-1}(1-z)^{\beta-1}, 
\mbox{\ \ \ where\ \ \ } B(\alpha, \beta) = \frac{\Gamma(\alpha) \Gamma(\beta)}{\Gamma(\alpha + \beta)}.$$
Thus
$$ \pr(Z \geq 1-\epsilon) 
\ = \ 
\frac{1}{B(\alpha, \beta)} \int_{1-\epsilon}^1 z^{\alpha-1} (1-z)^{\beta-1} dz 
\ \geq \ 
\frac{1}{B(\alpha, \beta)} \int_{1-\epsilon}^1 (1-z)^{\beta-1} dz 
\ = \ 
\frac{1}{B(\alpha, \beta)} \frac{\epsilon^\beta}{\beta}.$$
Likewise,
$$
\pr(Z \geq 1-\epsilon) 
\ \leq \ 
\frac{1}{B(\alpha, \beta)} (1-\epsilon)^{\alpha-1} \int_{1-\epsilon}^1 (1-z)^{\beta-1} dz 
\ = \ 
\frac{1}{B(\alpha, \beta)} \frac{\epsilon^\beta}{\beta} (1-\epsilon)^{\alpha-1}.
$$
When $\alpha = 1$, these bounds coincide and $B(\alpha,\beta) = \Gamma(\beta)/\Gamma(\beta + 1) = 1/\beta$.
\end{proof}

\subsection{Proof of Lemma~\ref{lemma:good-theta}}

Pick any good $\theta$, and let $x = \pi_M(\theta)$ denote its projection on $M$.

Recall that $\X$ consists of unit vectors. The points that lie in $C(\theta)$ are the $k/m$ fraction of $x$'s (under distribution $\mu$) that have the highest dot product with $\theta$, or equivalently, the $k/m$ fraction of $x$'s that are closest to $\theta$. Thus $C(\theta)$ is a set of the form $B(\theta, r')$, where radius $r'$ is chosen so that $\mu(B(\theta, r')) = k/m$. However, it is not necessarily of the form $B(x, r'')$, and this causes some complications. In particular, we need to address two questions: (i) if a point $x' \in M$ lies within distance $r < \rho$ of $x$, how far can be possibly be from $\theta$, and conversely, (ii) if $x' \in M$ lies within distance $r' < \rho$ of $\theta$, how far can it possibly be from $x$?

The condition on $M$'s reach plays a key role here. Consider the normal vector $\theta - x$ that connects $x$ to $\theta$. Let $\Delta = \|\theta - x\|$ denote its length, which is $< \rho/2$ by the goodness condition, and let unit vector $u = (\theta - x)/\|\theta - x\|$ be the corresponding direction. Consider balls of radius $\rho$ centered at $x + \rho u$ and $x - \rho u$, and touching at $x$ (Figure~\ref{fig:two-spheres}, left). By the definition of reach, $M$ does not touch the interiors of these two balls.

\begin{figure}
\begin{center}
\raisebox{-.25in}{\includegraphics[width=1.25in]{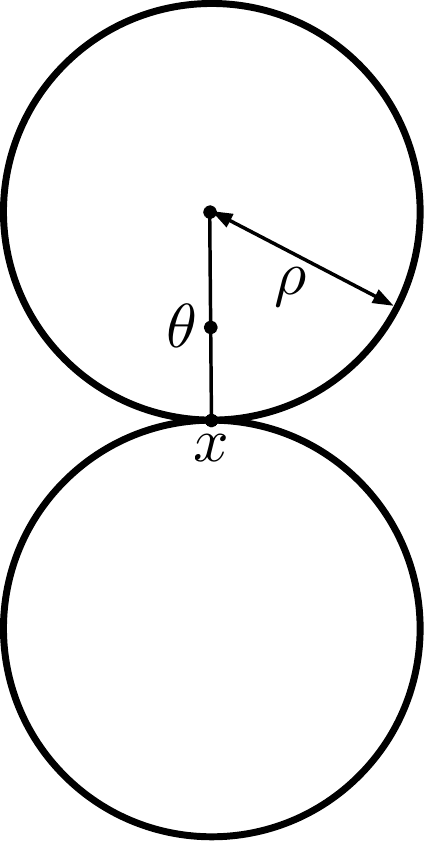}}
\hskip.75in
\raisebox{-.25in}{\includegraphics[width=1.25in]{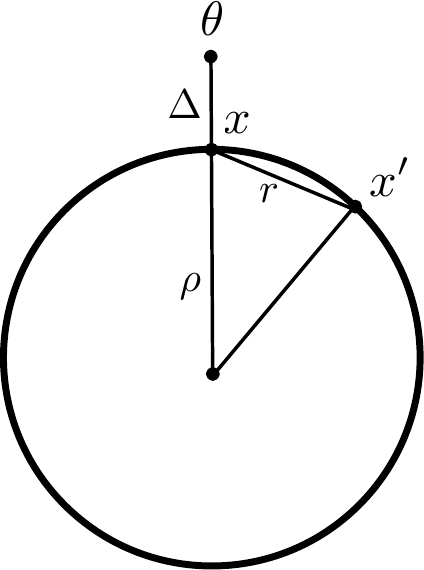}}
\hskip.5in
\raisebox{.8in}{\includegraphics[width=1.25in]{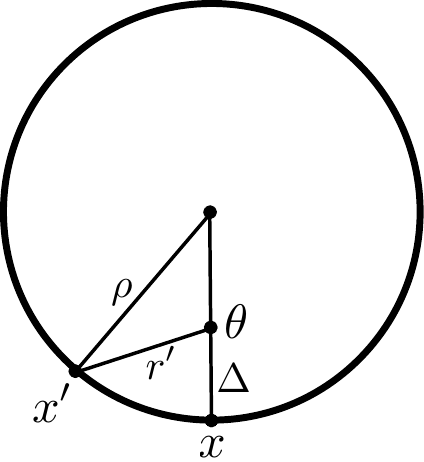}}
\end{center}
\caption{Left: Manifold $M$ does not touch the interior of these two balls of radius $\rho$. Middle and right: The two situations studied in the proof.}
\label{fig:two-spheres}
\end{figure}

Let's start with question (i). For any point $x' \in M$ at distance $r < \rho$ from $x$, consider the plane defined by $x$, $x'$, and $\theta$. The middle panel of Figure~\ref{fig:two-spheres} shows the furthest $x'$ could conceivably be from $\theta$. Applying the law of cosines twice, we have
\begin{align*}
r^2 &= 2\rho^2 - 2 \rho^2 \cos \phi \\
\|\theta - x'\|^2 &= (\rho + \Delta)^2 + \rho^2 - 2 (\rho + \Delta) \rho \cos \phi
\end{align*}
where $\phi$ is the angle subtended at the center of the circle by $x$ and $x'$. These imply
$\|\theta - x'\|^2 = \Delta^2 + r^2 (\rho + \Delta)/\rho$. Thus 
\begin{equation}
B_M(x,r) \ \subset \ B \left(\theta, \sqrt{\Delta^2 +  \frac{\rho + \Delta}{\rho} r^2} \right).
\label{eq:x-to-theta}
\end{equation}

Moving to (ii), pick any point $x' \in M$ at distance $r' < \rho$ from $\theta$; clearly, $r' \geq \Delta$. The right panel of Figure~\ref{fig:two-spheres} (depicting, once again, the plane defined by $x$, $x'$, and $\theta$) shows the furthest $x'$ could lie from $x$. Again using the law of cosines, and letting $\phi$ be the angle subtended at the center of the circle,
\begin{align*}
(r')^2 &= \rho^2 + (\rho - \Delta)^2 - 2 \rho(\rho - \Delta) \cos \phi \\
\|x-x'\|^2 &= 2 \rho^2 - 2 \rho^2 \cos \phi
\end{align*}
whereupon
$\|x - x'\|^2 = ((r')^2 - \Delta^2) \rho/(\rho-\Delta)$
and thus 
\begin{equation}
B_M(\theta, r') \subset B\left(x, \sqrt{\frac{\rho}{\rho-\Delta}((r')^2 - \Delta^2)}\right).
\label{eq:theta-to-x}
\end{equation}

To get the left-hand containment in the lemma statement, pick $r = \sqrt{(\rho-\Delta)/(\rho+\Delta)} (k/(c_2m))^{1/d_o}$ and use (\ref{eq:x-to-theta}) and (\ref{eq:theta-to-x}) to observe that for
\begin{align*}
r' &= \sqrt{\Delta^2 + \frac{\rho + \Delta}{\rho} r^2} \\
r'' &= \sqrt{\frac{\rho}{\rho-\Delta} ((r')^2 - \Delta^2)} = r \sqrt{\frac{\rho+\Delta}{\rho-\Delta}} = \left( \frac{k}{c_2 m} \right)^{1/d_o},
\end{align*}
we have $B_M(x,r)  \subset B_M(\theta, r') \subset B_M(x, r'')$. But by (\ref{eq:almost-uniform-mu}), $\mu(B(x, r'')) < k/m$. Thus $C(\theta) \supset B_M(\theta, r') \supset B_M(x,r)$.

The right-hand containment of the lemma proceeds similarly, by first observing that by (\ref{eq:almost-uniform-mu}), we have $\mu(B(x,r)) \geq k/m$ for $r = (k/c_1m)^{1/d_o}$. Taking
\begin{align*}
r' &= \sqrt{\Delta^2 + \frac{\rho + \Delta}{\rho} r^2} \\
r'' &= \sqrt{\frac{\rho}{\rho-\Delta} ((r')^2 - \Delta^2)} = \sqrt{\frac{\rho + \Delta}{\rho-\Delta}} \left( \frac{k}{c_1 m} \right)^{1/d_o},
\end{align*}
we have $B_M(x,r) \subset B_M(\theta, r') \subset B_M(x, r'')$. Thus $C(\theta) \subset B_M(\theta, r')$, and this is in turn contained in $B_M(x,r'')$.

\subsection{Proof of Lemma~\ref{lemma:good-theta-cover}}

Recall that $C(\theta)$ is the set of inputs $x$ that lie within the response region of $\theta$. We now need a dual notion, that of the $\theta$'s whose projection onto the manifold lies close to a specified $x$. For any $x \in M$ and $r > 0$, let
$$ A(x,r) = \{\theta \in \Gamma_{\rho/2}: \| \pi_M(\theta) - x \| < r \} .$$
Lemma~\ref{lemma:good-theta} tells us that if $m$ is large enough that $(k/c_1m)^{1/d_o} \leq \min(\rho, r_o)$, then for
$$r_1 = \frac{1}{2} \left( \frac{k}{c_2m} \right)^{1/d_o}$$
we have the implication $\theta \in A(x,r_1) \implies x \in C(\theta)$.

We first note that the regions $A(x,r)$ have non-negligible probability mass under $\nu$.
\begin{lemma}
Suppose $\nu$ is the multivariate Gaussian $N(0, \sigma^2 I_d)$. There is a constant $c_d$, that depends on the dimension $d$, such that for any $x \in M$ and $0 < r < r_o$, we have $\nu(A(x,r)) \geq c_d r^{d_o}$.
\label{lemma:volumes-of-A-cells}
\end{lemma}
\begin{proof}
Recall that $A(x,r)$ consists of all $\theta \in \Gamma_{\rho/2}$ that project onto $B_M(x,r)$. By (\ref{eq:volume-on-manifold}), the $d_o$-dimensional volume of $B_M(x,r)$ is $\geq c_3 r^{d_o}$. From~\cite[Lemma 18]{MMS16}, the $d$-dimensional volume of $A(x,r)$ can be bounded below as
$$ \mbox{vol}_d(A(x,r)) \geq \left( \frac{1}{2} \right)^{d_o} \left( \frac{\rho}{2} \right)^{d-d_o} \mbox{vol}_{d_o}(B_M(x,r)) \geq c_3 \rho^{d-d_o} r^{d_o}.$$

Since $M \subset S^{d-1}$, the set $A(x,r)$ lies within $B(0, 1+\rho/2)$. The smallest value that the density $\nu$ attains in this ball is
$$ \nu_o = \frac{1}{(2\pi)^{d/2} \sigma^d} \exp \left( - \frac{(1 + \rho/2)^2}{2 \sigma^2} \right) .$$
Therefore, $\nu(A(x,r)) \geq \nu_o \mbox{vol}_d(A(x,r)) \geq c_d r^{d_o}$, for some constant $c_d$ that scales exponentially with $d$.
\end{proof}

We can now embark upon the proof of Lemma~\ref{lemma:good-theta-cover}.

By (\ref{eq:almost-uniform-mu}), $M$ has a $(r_1/2)$-cover $\widehat{M}$ of size at most $(1/c_1) (4/r_1)^{d_o} \leq (c_2/c_1) 8^{d_o} m/k$. To see this, pick points $x_1, x_2, \ldots, x_N \in M$ that are at distance $> r_1/2$ from each other. The balls $B(x_i, r_1/4)$ are disjoint and each has $\mu(B(x_i, r_1/4)) \geq c_1 (r_1/4)^{d_o}$. Since the total probability mass of these balls is at most 1, it follows that $c_1 (r_1/4)^{d_o} N \leq 1$, giving an upper bound on $N$.

Pick any $\widehat{x} \in \widehat{M}$.
\begin{align*}
\mbox{Pr}(\mbox{no $\theta_j$ lies in $A(\widehat{x}, r_1/2)$}) 
&= (1 - \nu(A(\widehat{x}, r_1/2)))^m \ \leq \ (1 - c_d (r_1/2)^{d_o})^m \\
&= \left(1 - \frac{c_dk}{4^{d_o}c_2 m}\right)^m \leq \exp \left( - \frac{c_dk}{4^{d_o} c_2} \right) .
\end{align*}
For $k$ as specified in the lemma statement, with a suitable choice of $c_d'$, this is $< \delta/|\widehat{M}|$. We now take a union bound over $\widehat{M}$, to conclude that with probability at least $1-\delta$, for every $\widehat{x} \in \widehat{M}$, there is some good $\theta_j \in A(\widehat{x}, r_1/2)$.

Now pick an arbitrary $x \in M$. There is some $\widehat{x} \in \widehat{M}$ with $\|x - \widehat{x} \| \leq r_1/2$. Moreover, $\theta_j \in A(\widehat{x}, r_1/2) \implies \theta_j \in A(x, r_1) \implies x \in C(\theta_j)$.

\subsection{Proof of Theorem~\ref{thm:manifold-distrib-approx}}

The following result can be found in~\cite{BBL05}.
\begin{lemma}[Theorem 5.1~\cite{BBL05}]
\label{lemma:BBL-result}
Let $\F$ be a class of $\{0,1\}$-valued functions over $\X$, and let $\nu$ be a probability measure over $\X$. For any $\delta \in \{0, 1\}$, if $x_1, \ldots, x_n \sim \nu$, then with probability at least $1-\delta$,
\[ (\forall f \in \F )\frac{\E[f] - \E_n[f]}{\sqrt{\E_n[f]}} \ \leq \ 2\sqrt{\frac{ \log \left(S_{2n}(\F)\right) + \log(4/\delta)}{n}} \]
where $\E_n[f] = \frac{1}{n} \sum f(x_i)$ and $S_{n}(\F)$ denotes the $n$-th shattering coefficient of $\F$.
\end{lemma}

Note for any finite class of functions $\F$, we trivially have $S_{n}(\F) \leq |\F|$. Thus, the following is a corollary of Lemma~\ref{lemma:BBL-result}.

\begin{lemma}
\label{lemma:finite-ball-concentration}
Let $\B$ denote any finite collection of balls in $\R^d$, and let $\delta \in (0,1)$. If $\theta_1, \ldots, \theta_m$ are drawn i.i.d. from $\nu$, then with probability at least $1-\delta$, we have that any $B \in \B$ that satisfies 
\[ \nu(B) \ \geq \frac{k}{m} + \frac{2}{m} \sqrt{k \left( \log(|\B| ) + \log(4/\delta) \right)} \]
contains at least $k$ of the $\theta_i$.
\end{lemma}

The $r$-covering number of a set $M$, denoted $C(M,r)$, is the minimum number of balls of radius $r$, centered at points in $M$, needed to cover $M$.  A related notion is the $r$-packing number, denoted $P(M,r)$, which is the maximum number of balls of radius $r/2$, centered at points in $M$, that do not overlap. A well-known fact is that $|C(M,r)| \leq |P(M,r)|$. The following lemma shows that one consequence of \eqref{eqn:measure-lower-bound} is a bound on the covering number of $M$.

\begin{lemma}
\label{lem:manifold-covering-number}
For any $c_o, r_o > 0$ obeying \eqref{eqn:measure-lower-bound} and $r < 2r_o$, $C(M, r) \leq  (2/r)^{d_o}/c_o$.
\end{lemma}
\begin{proof}
Let $\P$ denote an $r$-packing of $M$. Then there exists an $r$-covering $\CC$ of $M$ such that $|\CC| \leq |\P|$. Moreover, since $\nu$ is a probability measure, we have
\[ 1 \ = \ \nu(M) \ \geq \ \sum_{B \in \P} \nu(B) \  \geq \ |\P| c_o (r/2)^{d_o} \ \geq \  |\CC| c_o (r/2)^{d_o}. \]
Rearranging completes the proof.
\end{proof}

We are now ready to prove Theorem~\ref{thm:manifold-distrib-approx}.

Take $r$ to be
\[ r \ = \ \frac{4}{{c_o}^{1/d_o}}\left(\frac{2k}{m} + \frac{d_o \log(8m/k) + \log(4/\delta)}{m}\right)^{1/d_o}. \] 
By assumption, $r \leq r_o$. Now take $\B$ to be a minimal $r/2$-covering of $M$. By Lemma~\ref{lem:manifold-covering-number}, we know $|\B| \leq  (4/r)^{d_o}/c_o$. Moreover, for each $\theta_j$, there is some $B \in \B$ satisfying $B \subset B(\theta_j, r)$, meaning
\[ \nu(B(\theta_j, r)) \ \geq \ \nu(B) \ \geq \ \frac{2k}{m} +  \frac{d_o \log(8m/k) + \log(4/\delta)}{m} \ \geq \ \frac{k}{m} + \frac{2}{m} \sqrt{k \left( \log(|\B| ) + \log(4/\delta) \right)}, \]
where the last inequality follows from the AM-GM inequality and substitution. By Lemma~\ref{lemma:finite-ball-concentration}, we have that with probability at least $1-\delta$,  each $B(\theta_j, r)$ contains $k$ of the $\theta$'s. In other words, $C_j \subset B(\theta_j, r)$ for all $j$. Applying Lemma~\ref{lemma:wta-approx-bd} gives the theorem.

\subsection*{Acknowledgements}
SD is grateful to the National Science Foundation for support under grant CCF-1813160 and to the Institute for Advanced Study for hosting him in Fall 2019, when part of this work was done. CT also acknowledges support from the NSF under grant CCF-1740833.

\bibliographystyle{abbrv}
\bibliography{representation}

\begin{thebibliography}{10}

\bibitem{BS14}
B.~Babadi and H.~Sompolinsky.
\newblock Sparseness and expansion in sensory representations.
\newblock {\em Neuron}, 83:1213--1226, 2014.

\bibitem{B93}
A.~Barron.
\newblock Universal approximation bounds for superpositions of a sigmoidal
  function.
\newblock {\em IEEE Transactions on Information Theory}, 39(3):930--945, 1993.

\bibitem{BBL05}
S.~Boucheron, O.~Bousquet, and G.~Lugosi.
\newblock Theory of classification: A survey of some recent advances.
\newblock {\em ESAIM: Probability and Statistics}, 9:323--375, 2005.

\bibitem{CRT06}
E.~Candes, J.~Romberg, and T.~Tao.
\newblock Stable signal recovery from incomplete and inaccurate measurements.
\newblock {\em Communications on Pure and Applied Mathematics},
  59(8):1207--1223, 2006.

\bibitem{CRAA13}
S.~Caron, V.~Ruta, L.~Abbott, and R.~Axel.
\newblock Random convergence of olfactory inputs in the {D}rosophila mushroom
  body.
\newblock {\em Nature}, 497:113--117, 2013.

\bibitem{CLM11}
M.~Chacron, A.~Longtin, and L.~Maler.
\newblock Efficient computation via sparse coding in electrosensory neural
  networks.
\newblock {\em Current Opinion in Neurobiology}, 21:752--760, 2011.

\bibitem{CD10}
K.~Chaudhuri and S.~Dasgupta.
\newblock Rates of convergence for the cluster tree.
\newblock In {\em Advances in Neural Information Processing Systems}, 2010.

\bibitem{CD14}
K.~Chaudhuri and S.~Dasgupta.
\newblock Rates of convergence for nearest neighbor classification.
\newblock In {\em Advances in Neural Information Processing Systems}, 2014.

\bibitem{CH67}
T.~Cover and P.~Hart.
\newblock Nearest neighbor pattern classification.
\newblock {\em IEEE Transactions on Information Theory}, 13:21--27, 1967.

\bibitem{C89}
G.~Cybenko.
\newblock Approximation by superpositions of a sigmoidal function.
\newblock {\em Mathematics of Control, Signals and Systems}, 2(4):303--314,
  1989.

\bibitem{DSSN18}
S.~Dasgupta, T.~Sheehan, C.~Stevens, and S.~Navlakha.
\newblock A neural data structure for novelty detection.
\newblock {\em Proceedings of the National Academy of Sciences},
  115(51):13093--13098, 2018.

\bibitem{DSN17}
S.~Dasgupta, C.~Stevens, and S.~Navlakha.
\newblock A neural algorithm for a fundamental computing problem.
\newblock {\em Science}, 358:793--796, 2017.

\bibitem{DGL96}
L.~Devroye, L.~Gyorfi, and G.~Lugosi.
\newblock {\em A Probabilistic Theory of Pattern Recognition}.
\newblock Springer, 1996.

\bibitem{D06}
D.~Donoho.
\newblock Compressed sensing.
\newblock {\em IEEE Transactions on Information Theory}, 52(4):1289--1306,
  2006.

\bibitem{F89}
K.-I. Funahashi.
\newblock On the approximate realization of continuous mappings by neural
  networks.
\newblock {\em Neural Networks}, 2(3):183--192, 1989.

\bibitem{HSW89}
K.~Hornik, M.~Stinchcombe, and H.~White.
\newblock Multilayer feedforward networks are universal approximators.
\newblock {\em Neural Networks}, 2(5):359--366, 1989.

\bibitem{K09b}
P.~Kanerva.
\newblock Hyperdimensional computing: An introduction to computing in
  distributed representation with high-dimensional random vectors.
\newblock {\em Cognitive Computation}, 1(2):139--159, 2009.

\bibitem{LBCLM14}
A.~Lin, A.~Bygrave, A.~de~Calignon, T.~Lee, and G.~Miesenbock.
\newblock Sparse, decorrelated odor coding in the mushroom body enhances
  learned odor discrimination.
\newblock {\em Nature Neuroscience}, 17(4), 2014.

\bibitem{MMS16}
M.~Maggioni, S.~Minsker, and N.~Strawn.
\newblock Multiscale dictionary learning: non-asymptotic bounds and robustness.
\newblock {\em Journal of Machine Learning Research}, 17(1):43--93, 2016.

\bibitem{MTJ09}
N.~Masse, G.~Turner, and G.~Jefferis.
\newblock Olfactory information processing in {D}rosophila: review.
\newblock {\em Current Biology}, 19:R700--R713, 2009.

\bibitem{NSW06}
P.~Niyogi, S.~Smale, and S.~Weinberger.
\newblock Finding the homology of submanifolds with high confidence from random
  samples.
\newblock {\em Discrete and Computational Geometry}, 2006.

\bibitem{OBW10}
S.~Olsen, V.~Bhandawat, and R.~Wilson.
\newblock Divisive normalization in olfactory population codes.
\newblock {\em Neuron}, 66(2):287--299, 2010.

\bibitem{OF04}
B.~Olshausen and D.~Field.
\newblock Sparse coding of sensory inputs.
\newblock {\em Current Opinion in Neurobiology}, 14:481--487, 2004.

\bibitem{PV19}
C.~Papadimitriou and S.~Vempala.
\newblock Random projection in the brain and computation with assemblies of
  neurons.
\newblock In {\em Innovations in Theoretical Computer Science}, 2019.

\bibitem{RR07}
A.~Rahimi and B.~Recht.
\newblock Random features for large-scale kernel machines.
\newblock In {\em Advances in Neural Information Processing Systems}, 2007.

\bibitem{SA09}
D.~Stettler and R.~Axel.
\newblock Representations of odor in the piriform cortex.
\newblock {\em Neuron}, 63:854--864, 2009.

\bibitem{S77}
C.~Stone.
\newblock Consistent nonparametric regression.
\newblock {\em Annals of Statistics}, 5:595--645, 1977.

\bibitem{S80}
C.~Stone.
\newblock Optimal rates of convergence for nonparametric estimators.
\newblock {\em Annals of Statistics}, 8(6):1348--1360, 1980.

\bibitem{TBL08}
G.~Turner, M.~Bazhenov, and G.~Laurent.
\newblock Olfactory representations by {D}rosophila mushroom body neurons.
\newblock {\em J. Neurophysiol.}, 99:734--746, 2008.

\bibitem{W13}
R.~Wilson.
\newblock Early olfactory processing in {D}rosophila: Mechanisms and
  principles.
\newblock {\em Annual Review of Neuroscience}, 36:217--241, 2013.

\end{thebibliography}

\end{document}